\documentclass[a4paper,10pt]{scrartcl}
\usepackage[utf8]{inputenc}
\usepackage{amsthm,amsmath}
\usepackage{amssymb,amsmath}
\usepackage{todonotes}

\newtheorem{lemma}{Lemma}

\newtheorem{claim}{Claim}
\newtheorem{theorem}{Theorem}

\title{Parameterized Compilation Lower Bounds for Restricted CNF-formulas}
\author{Stefan Mengel\\ CNRS, CRIL UMR 8288, France}
\begin{document}

\maketitle

\begin{abstract}
We show unconditional parameterized lower bounds in the area of knowledge compilation, more specifically on the size of circuits in decomposable negation normal form (DNNF) that encode CNF-formulas restricted by several graph width measures. In particular, we show that
\begin{itemize}
\item there are CNF formulas of size $n$ and modular incidence treewidth $k$ whose smallest DNNF-encoding has size $n^{\Omega(k)}$, and
\item there are CNF formulas of size $n$ and incidence neighborhood diversity $k$ whose smallest DNNF-encoding has size $n^{\Omega(\sqrt{k})}$.
\end{itemize}
These results complement recent upper bounds for compiling CNF into DNNF and strengthen---quantitatively and qualitatively---known conditional low\-er bounds for cliquewidth. Moreover, they show that, unlike for many graph problems, the parameters considered here behave significantly differently from treewidth. 
\end{abstract}

\section{Introduction}

Knowledge compilation is a preprocessing regime that aims to translate or ``compile'' knowledge bases, generally encoded as CNF formulas, into different representations more convenient for a task at hand. The idea is that many queries one would like to answer on the knowledge base, say clause entailment queries, are intractable in CNF encoding, but tractable for other representations. When there are many queries on the same knowledge base, as for example in product configuration, it makes sense to invest into a costly preprocessing to change the representation once in order to then speed up the queries and thus amortize the time spent on the preprocessing.

One critical question when following this approach is the choice of the representation that the knowledge is encoded into. In general, there is a trade-off between the usefulness of a representation (which queries does it support efficiently?) and succinctness (what is the size of the encoded knowledge base?). This trade-off has been studied systematically~\cite{DarwicheM02}, leading to a fine understanding of the different representations. In particular, circuits in decomposable negation normal form (short DNNF)~\cite{Darwiche01} have been identified as a representation that is more succinct than nearly all other representations while still allowing useful queries. Consequently, DNNFs play a central role in knowledge compilation.

This paper should be seen as complementing the findings of~\cite{BovaCMS15}: In that paper, algorithms compiling CNF formulas with restricted underlying graph structure were presented, showing that popular graph width measures like treewidth and cliquewidth can be used in knowledge compilation. More specifically, every CNF formula of incidence \emph{treewidth} $k$ and size $n$ can be compiled into a DNNF of size $2^{O(k)} n$. Moreover, if $k$ is the incidence \emph{cliquewidth}, the size bound on the encoding becomes $n^{O(k)}$. As has long been observed, $2^{O(k)} n$ is of course far preferable to $n^{O(k)}$ for nontrivial sizes of $n$---in fact, this is the main premise of the field of parameterized complexity theory, see e.g.~\cite{FlumG06}. Consequently, the results of~\cite{BovaCMS15} leave open the question if the algorithm for clique-width based compilation of CNF formulas can be improved. 

In fact, the paper~\cite{BovaCMS15} already gives a partial answer to this question, proving that there is no compilation algorithm achieving fixed-parameter compilability, i.e., a size bound of $f(k) p(|F|)$ for a function $f$ and a polynomial $p$. But unfortunately this result is based on the plausible but rather non-standard complexity assumption that not all problem in $\mathsf{W}[1]$ have $\mathsf{FPT}$-size circuits. The result of this paper is that this assumption is not necessary. We prove a lower bound of $|F|^{\Omega(k)}$ for formulas of \emph{modular incidence treewidth} $k$ where modular treewidth is a restriction of cliquewidth proposed in~\cite{PaulusmaSS13}. It follows that the result in~\cite{BovaCMS15} is essentially tight. Moreover, we show a lower bound of $|F|^{\Omega(\sqrt{k})}$ for formulas of neighborhood diversity~$k$~\cite{Lampis10}. This intuitively shows that all graph width measures that are stable under adding modules, i.e., adding a new vertex that has exactly the same neighborhood as an existing vertex, behave qualitatively worse than treewidth for compilation into DNNFs.

\paragraph{Related work.}

Parameterized knowledge compilation was first introduced by Chen~\cite{Chen05} and has seen some recent renewed interest, see e.g.~\cite{Chen15,Haan15} for work on conditional lower bounds. Unconditional lower bounds based on treewidth can e.g.~be found in \cite{Razgon14,Razgon14b}, but they are only for different versions of branching programs that are known to be less succinct than DNNF. Moreover, these lower bounds fail for DNNFs as witnessed by the upper bounds of~\cite{BovaCMS15}.

There is a long line of research using graph and hypergraph width measures for problems related to propositional satisfiability, see e.g.~the extensive discussion in~\cite{Brault-BaronCM15}. The paper~\cite{OrdyniakPS13} gave the first parameterized lower bounds on SAT with respect to graph width measures, in particular cliquewidth. This result was later improved to modular treewidth to complement an upper bound for model counting~\cite{PaulusmaSS13} and very recently to neighborhood diversity~\cite{Dell0LMM15}, a width measure introduced in~\cite{Lampis10}. We remark that the latter result could be turned into a conditional parameterized lower bound similar to that in~\cite{BovaCMS15} discussed above.

Our lower bounds strongly rely on the framework for DNNF lower bounds proposed in \cite{IJCAIversion} and communication theory lower bounds from~\cite{DurisHJSS04}, for more details see Section~\ref{sct:statement}.

\section{Preliminaries}

% \paragraph{Coding theory.}

In the scope of this paper, a \emph{linear code} $C$ is the solution of a system of linear equations $A\bar x = 0$ over the boolean field $\mathbb{F}_2$. The matrix $A$ is called the \emph{parity-check matrix} of $C$. The characteristic function $f_C$ is the boolean function that, given a boolean string $e$, evaluates to~$1$ if and only if $e$ is in $C$.
% We call the parity-check matrix $H$ of a linear code $s$-good

We use the notation $[n]:=\{1, \ldots, n\}$ and $[n_1, n_2]:=\{ n_1, n_1+1, \ldots, n_2\}$ to denote integer intervals.
We use standard notations from graph theory and assume the reader to have a basic background in in the area~\cite{Diestel12}. By $N(v)$ we denote the open neighborhood of a vertex in a graph.

We say that two vertices $u$, $v$ in a graph $G=(V,E)$ have the same neighborhood type if and only if $N(u)\setminus \{v\} = N(v)\setminus \{u\}$. It can be shown that having the same neighborhood type is an equivalence relation on $V$~\cite{Lampis10}. The neighborhood diversity of $G$ is defined to be the number of equivalence classes of $V$ with respect to neighborhood types.

A generalization of neighborhood diversity is \emph{modular treewidth} which is defined as follows: From a graph $G$ we construct a new graph $G'$ by contracting all vertices sharing a neighborhood type, i.e., from every equivalence class we delete all vertices but one. The modular treewidth of $G$ is then defined to be the treewidth of $G'$\footnote{Note that the definition in~\cite{PaulusmaSS13} differs from the one we give here, but can easily be seen to be equivalent for bipartite graphs and thus incidence graphs of CNF formulas. We keep our definition to be more consistent with the definition of neighborhood diversity.}. Modular pathwith is defined in the obvious analogous way.

% \todo{pathwidth}

% \paragraph{CNF.} 
We assume basic familiarity with propositional logic and in particular CNF formulas. We define the size of a CNF formula to be the overall number of occurrences of literals, i.e., the sum of the sizes of the clauses where the size of a clause is the number of literals in the clause. The incidence graph of a CNF formula $F$ has as vertices the variables and clauses of $F$ and an edge between every clause and the vertices contained in it. The projection of an assignment $a:X\rightarrow \{0,1\}$ to a set $Z$ is the restriction of $a$ to the variable set $Z$. This definition generalizes to sets of assignments in the obvious way. Moreover, the projection of a boolean function $f$ on $X$ to $Z$ is defined as the boolean function on $Z$ whose satisfying assignments are those of $f$ projected to $Z$. For the width measures introduced above, we define the with of a formula to be that of its incidence graph.

\section{Statement of the Main Results and Preparation of the Proof}\label{sct:statement}

We now state our main results. The first theorem shows that modular pathwith---and thus also more general parameters like cliquewidth and modular treewidth---do not allow fixed-parameter compilation to DNNF.

\begin{theorem}\label{thm:maintw}
For every $k$ and for every $n$ big enough there is a CNF formula $F$ of size at most $n$ and modular pathwidth $k$ such that any DNNF computing the same function as $F$ must have size $n^{\Omega(k)}$. 
\end{theorem}

We also show lower bounds for neighborhood diversity that are nearly as strong as those for modular pathwidth.

\begin{theorem}\label{thm:mainnd}
For every $k$ and for every $n$ big enough there is a CNF formula $F$ of size polynomial in $n$ and with neighborhood diversity $k$ such that any DNNF computing the same function as $F$ must have size $n^{\Omega(\sqrt{k})}$. 
\end{theorem}

At this point, the attentive reader may be a little concerned because we promise to prove lower bounds for DNNF which we have not even defined in the preliminaries. In fact, it is the main strength of the approach in~\cite{IJCAIversion} that the definition and properties of DNNF are not necessary to show our lower bounds, because we can reduce showing lower bounds on DNNF to a problem in communication complexity. Since we will not use any properties of DNNF, we have decided to leave out the definition 
% for space reasons and refer to e.g.~\cite{DarwicheM02}.
out of the main text; the interested reader may find a short overview in Appendix~\ref{app:DNNF}. 
Here we will only use the following result.

\begin{theorem}[\cite{IJCAIversion}]\label{thm:communicationDNNF}
 Let $f$ be a function computed by a DNNF of size $s$. Then $f$ has a multi-partition rectangle cover of size $s$.
\end{theorem}

Now the reader might be a little puzzled about what multi-partition rectangle covers of a function are. Since we will also only use them as a black box in our proofs and do not rely on any of their properties, we have opted to leave out their definition 
% and refer to~\cite{DurisHJSS04} 
the main text. The curious reader may find a very short introduction into the beautiful area of multi-partition communication complexity in Appendix~\ref{app:CC}.

We will use a powerful theorem which follows directly from the results in~\cite{DurisHJSS04}.

% \begin{lemma}\cite{DurisHJSS04}\label{lem:communicationlower}
%  Let $C$ be a binary linear code with an $s$-good $m\times n$ parity-check matrix $H$ and characteristic function $f_C$. Then every multi-partition rectangle cover of $f_C$ has size at least $2^{2s-m}$.
% \end{lemma}
% 
% \begin{proposition}\cite{DurisHJSS04}\label{prop:constructH}
%  Let $m\le n/32$. Let $A$ be a random boolean $m\times n$-matrix. Then $H$ is $(m-1)$-good with probability $1-2^{\Omega(n)}$.
% \end{proposition}

\begin{theorem}[\cite{DurisHJSS04}]\label{thm:combinecodes}
 For every $n'\in \mathbb{N}$ and every $m'\le n'/32$ there is a linear code $C$ with a $m'\times n'$ parity check matrix such that every multi-partition rectangle cover of the characteristic function $f_C$ has size at least $\frac{1}{4}2^{m'}$.
\end{theorem}

\section{Accepting Codes by CNF Formulas}

In this section we will construct CNF formulas to accept linear codes. We will first start with a naive encoding that will turn out to be of unbounded modular treewidth and thus not directly helpful to us. We will then show how to change the encoding in such a way that the modular treewidth and even the neighborhood diversity are small and the size of the resulting CNF is small enough to show meaningful lower bounds for encodings in DNNF with Theorem~\ref{thm:combinecodes}.

\subsection{The naive approach}\label{sct:naive}

In this subsection, we show how we can check $m$ linear equations on variables $x_1, \ldots, x_n$ efficiently by CNF. The idea is to simply consider one variable after the other and remember the parity for the equations at hand. To this end, fix an $m\times n$ matrix $A=(a_{ij})$. We denote the resulting equations of the system $A\bar{x}=0$ by $E_1, \ldots, E_m$. For each equation $E_i$ we introduce variables $z_{ij}$ for $j\in [n]$ which intuitively remembers the parity of $E_i$ up to seeing the variable $x_j$.

We encode the computations for each $E_i$ individually: Introduce constraints \begin{align}\label{eq:1}a_{i,1}x_1 = z_{i,1},\\\label{eq:2}z_{i,j-1} + a_{ij} x_j = z_{ij}.\end{align} Note that $z_{i,n}$ yields the parity for equation $E_i$ which can then be checked for~$0$. This yields a system whose accepted inputs projected to the $x_i$ are the code words of the considered code. The constraints have all at most $3$ variables, so we can encode them into CNF easily.

Unfortunately, the resulting CNF can be shown to have high modular tree\-width, so it is not useful for our considerations. We will see how to reduce the modular treewidth and the neighborhood diversity of the system without blowing up the size of the resulting CNF-encoding too much.

\subsection{Bounding modular treewidth}

The idea for decreasing the modular treewidth is to not encode all constraints on the parities individually but combine them into larger constraints. So fix $n$ and $k$ and set $m:=k\log(n)$.
For each $j$, we will combine the constraints from (\ref{eq:1}) and (\ref{eq:2}) for blocks of $\log(n)$ values of $i$ into one. The resulting constraints are \begin{align*}R_1^\ell(x_1, z_{\ell\log(n) + 1,1}, \ldots, z_{(\ell+1)\log(n),1}) := \{ (d_1, t_{\ell\log(n) + 1,1}, \ldots, t_{(\ell+1)\log(n),1 }) \mid \\a_{i,1}d_1 = t_{i,1}, i=\ell\log(n)+1, \ldots, (\ell+1)\log(n)\}\end{align*}
and 
\begin{align*}R_j^\ell(x_i, z_{\ell\log(n) + 1,j-1}, \ldots, z_{(\ell+1)\log(n),j-1}, z_{\ell\log(n) + 1,j-1}, \ldots, z_{(\ell+1)\log(n),j-1}) := \\
\{ (d_i, t_{\ell\log(n) + 1,j-1}, \ldots, t_{(\ell+1)\log(n),j-1 }, t_{\ell\log(n) + 1,j}, \ldots, t_{(\ell+1)\log(n),j }) \mid \\t_{i,j-1} + a_{ij}d_j = t_{i,j}, i=\ell\log(n)+1, \ldots, (\ell+1)\log(n)\}\end{align*}
for $\ell = 0, \ldots, k-1$.

Note that the constraints $R_{j}^\ell$ have at most $2\log(n)+1$ boolean variables, so we can encode them into CNF of quadratic size where every clause contains all variables of $R_{j}^\ell$. Moreover, the $R_{j}^\ell$ encode all previous constraints from~(\ref{eq:1}) and~(\ref{eq:2}), so the assignments satisfying all $R_{j}^\ell$ projected to the $x_i$ still are exactly the code words of the code we consider. Call the resulting CNF $F$.

\begin{claim}
 $F$ has modular pathwidth at most $2k-1$.
\end{claim}
\begin{proof}
 Note that the clauses introduced when translating the constraint $R_{j}^\ell$ into CNF have by construction all the same set of variables. Thus these clauses have the same neighborhood type, and we can for modular treewidth restrict to an instance just having one clause for each $R_{j}^\ell$. We call the resulting vertex in the incidence graph $r_{\ell,j}$. Next, observe that the variables $z_{\ell\log(n) + i, j}$ and $z_{\ell\log(n) + i', j}$ for $i, i'\in [\log(n)]$ appear in exactly the same clauses. Thus these variables have the same neighborhood type as well, so we can delete all but one of them, say $z_{\ell\log(n),j}$ for $\ell=1, \ldots, k$. Call the resulting vertices in the incidence graph $s_{\ell, j}$.
 
 The resulting graph $G=(V,E)$ has \begin{align*}V=&\{x_j, s_{\ell, j}, r_{\ell,j}\mid j\in [n],\ell\in [k]\}.\\
 E=& \{x_jr_{\ell, j}\mid j\in [n], \ell\in [k]\}\cup \{s_{\ell, j-1}r_{\ell, j} \mid j\in [2, n], \ell\in [k]\} \\&\cup \{s_{\ell, j}r_{\ell, j} \mid j\in [n], \ell\in [k]\}\end{align*}

 We construct a path decomposition of $G$ as follows: The bags are the sets \begin{align*}B_2 := &\{x_1r_{1, 1}, \ldots, r_{\ell,1}\}, \\B_3:= &\{s_{1, 1}, \ldots, s_{\ell,1}, r_{1, 1}, \ldots, r_{\ell,1}\}\end{align*} and for $j=2, \ldots n$ \begin{align*}B_{3j-2} := &\{s_{1, j-1}, \ldots, s_{\ell,j-1}, r_{1, j}, \ldots, r_{\ell,j}\},\\ B_{3j-1} := &\{x_j, r_{1, j}, \ldots, r_{\ell,j}\}, \text{and} \\B_{3j} := &\{s_{1, j}, \ldots, s_{\ell,j}, r_{1, j}, \ldots, r_{\ell,j}\}.\end{align*}
 Ordering the bags $B_j$ with respect to their index yields a path decomposition of $G$ of width $2k-1$.
\end{proof}

Let us collect the results of the section into one statement.

\begin{lemma}\label{lem:encodecode}
 For every linear code $C$ with a $k\log(n)\times n$ parity check matrix there is a CNF formula $F$ in variable sets $X$ and $Z$ such that 
 \begin{itemize}
  \item the solution set of $F$ projected to $X$ is exactly $C$,
  \item $F$ has size $O(kn^3\log(n)^2)$, and
  \item $F$ has modular pathwidth at most $2k-1$.
 \end{itemize}
\end{lemma}
\begin{proof}
 It remains only to show the size bound on $F$. Note that we have $n$ variables $x_j$ and $kn\log(n)$ variables $z_{i,j}$. Moreover, we have $kn\log(n)$ constraints $R_i^\ell$. Each of those has $2\log(n)+1$ variables, so it can be encoded by $O(n^2)$ clauses with $O(\log(n))$ variables each. This yields $O(kn^3\log(n))$ clauses with $O(\log(n))$ variables. Consequently, the overall size of $F$ is $O(k n^3\log(n)^2)$.
\end{proof}

\subsection{Bounding neighborhood diversity}

Fix now two positive integers $N$ and $k$ and let $n:=32 k \log(N)$ and $m:=k\log(N)$ and consider $A$ with these parameters as before. We want to encode the code of~$A$ by a CNF with neighborhood diversity $O(k^2)$.

To do so, we split the variables $z_{ij}$ into $O(k^2)$ sets of $\log(N)^2$ variables $S_{rs}:= \{z_{ij} \mid i\in [r \log(N)+1, (r+1)\log(N)], j\in [s \log(N)+1, (s+1)\log(N)-1]\}$ for $r\in [k]$ and $s\in [32 k]$. Now create for all $r\in [k]$, $s\in [32k-1]$ a constraint $R_{rs}$ in the variables $X:=\{x_1, \ldots, x_n\}\cup S_{rs}\cup S_{r+1,s}$ that accepts all assignments to its variables that satisfy all constraints from \ref{eq:1} and \ref{eq:2} whose variables are variables of $R_{rs}$. Note that the resulting constraints cover all constraints of Section~\ref{sct:naive}, so we still accept the code defined by $A$ after projection to $X$.

The problem now is that, since we have $\Theta(\log(N)^2)$ boolean variables in each constraint, the resulting encoding into CNF could be superpolynomial in $N$ and thus to big for our purposes. This is easily repaired by the observation that fixing the values of $x_i$, $z_{i, s\log(N)}$ and $z_{i, s\log(N)+1}$ determines the values of the other $z_{ij}$ in all satisfying assignments. Consequently, we can project out these variables of the individual constraints without changing the accepted assignments to $X$. Call the resulting constraints $R'_{rs}$. It is easy to see that every constraint $R'_{rs}$ has only $O(\log(N))$ variables, so the CNF encoding in which every variable of $R_{r,s}'$ appears in every clause has polynomial size in $N$.

We claim that the resulting CNF $F$ has neighborhood diversity $O(k^2)$. To see this, note that the clauses introduced in the encoding of a fixed $R'_{rs}$ all have the same variables. It follows that the clause vertices in the incidence graph have $O(k^2)$ neighborhood types. The variables in $X$ all appear in all clauses, so they are all of the same neighborhood type. Finally, the vertices in each $S_{rs}$ appearing in an $R_{rs}'$ all appear in the same clauses, so they have $O(k^2)$ neighborhood types as well. This show that the incidence graph of the CNF formula $F$ has neighborhood diversity $O(k^2)$.

Let us again combine the results of this section into one summary statement.

\begin{lemma}\label{lem:encodecode2}
 For every linear code $C$ with a $k\log(N)\times 32 k\log(N)$ parity check matrix there is a CNF formula $F$ in variable sets $X$ and $Z$ such that 
 \begin{itemize}
  \item the solution set of $F$ projected to $X$ is exactly $C$,
  \item $F$ has size polynomial in $N$ and $k$, and
  \item $F$ has neighborhood diversity $O(k^2)$.
 \end{itemize}
\end{lemma}

\section{Completing the Proof}

We now combine the results of the previous sections to get our main results.

\begin{proof}[of Theorem~\ref{thm:maintw}]
 Let $C$ be a linear code as in Theorem~\ref{thm:combinecodes} with parameters $n'=n^{\frac{1}{4}}$ and $m':=\log(n) \frac{k}{2}= k \log(n^{\frac{1}{4}})$. Then by Theorem~\ref{thm:combinecodes} we know that every rectangle cover of the characteristic function $f_C$ has size at least $\frac{1}{4}2^{m'}= n^{\Omega(k)}$. 
 
 Now apply Lemma~\ref{lem:encodecode} to $C$ to get a CNF-formula $F$ of size less than $n$ and modular pathwidth less than $k$. Let $D$ be a DNNF representation of $F$ of minimal size $s$. Since DNNFs allow projection to a subset of variables without any increase of size~\cite{Darwiche01}, this yields a DNNF of size $s$ computing $f_C$. But then by Theorem~\ref{thm:communicationDNNF}, we get that $s\ge  n^{\Omega(k)}$.
\end{proof}
With the same proof but other parameters we get Theorem~\ref{thm:mainnd} from Lemma~\ref{lem:encodecode2}.

\section{Connections to Model Counting and Affine Decision Trees}

In this section we discuss connections of the findings of this paper to practical model counting. It has been shown that there is a tight connection between compilation and model counting, as runs of exhaustive DPLL-based model counting algorithms can be translated into (restricted) DNNFs~\cite{HuangD05}. Here the size of the resulting DNNF corresponds to the runtime of the model counter. Since state of the art solvers like Cachet~\cite{SangBBKP04} and sharpSAT~\cite{Thurley06} use exhaustive DPLL, the lower bounds in this paper can be seen as lower bounds for these programs: model counting for CNF formulas of size $n$ and, modular treewidth will take time at least $n^{\Omega(k)}$ when solved with these state-of-the-art solvers even with perfect caching and optimal branching variable choices. Note that in the light of the general conditional hardness result of~\cite{PaulusmaSS13} this is not surprising, but here we get concrete and unconditional lower bounds for a large class of algorithms used in practice. Naturally, we also directly get lower bounds for approaches that are based on compilation into DNNF as those in~\cite{Darwiche04,MuiseMBH12}, so we have lower bounds for most practical approaches to model counting.

One further interesting aspect to observe is that, while the instances that we consider are in a certain sense hard for practical model counting algorithms, in fact counting their models is extremely easy. Since we just want to count the number of solutions of a system of linear equations, basic linear algebra will do the job. A similar reasoning translated to compilation is the background for the definition of affine decision trees (ADT)~\cite{KoricheLMT13}, a compilation language that intuitively has checking an affine equation as a built-in primitive. Consequently, it is very easy to see that ADTs allow a very succinct compilation of the CNF formulas we consider in this paper. It follows, by setting the right parameters, that there are formulas where ADTs are exponentially more succinct than DNNF. We remark that this superior succinctness can also be observed in experiments when compiling the formulas of Section~\ref{sct:naive} with the compiler from~\cite{KoricheLMT13}.

\section{Conclusion}

We have shown that parameters like cliquewidth, modular treewidth and even neighborhood diversity behave significantly differently from treewidth for compilation into DNNF by giving lower bounds complementing the results of~\cite{BovaCMS15}. These unconditional lower bounds confirm conditional ones that had been known for some time already and improve them quantitatively. Our proofs heavily relied on the framework proposed in~\cite{IJCAIversion} thus witnessing the strength of this approach.
We have also discussed implications for practical model counting.

One consequence of our results is that most graph width measures that allow dense incidence graphs for the input CNF---like modular treewidth or cliquewidth and unlike treewidth which forces a small number of edges---do not allow fixed-parameter compilation into DNNF. A priori, there is no reason why many edges in the incidence graphs, which translates into many big clauses, should necessarily make compilation hard. Thus it would be interesting to see if there are any width measures that allow dense graphs and fixed-parameter compilation at the same time. One width measure that might be worthwhile analyzing is the recently defined measure sm-width~\cite{SaetherT14}.

\paragraph{Acknowledgments.} The author would like to thank Florent Capelli for helpful discussions. Moreover, he thanks Jean-Marie Lagniez for helpful discussions and for experiments with the compiler from~\cite{KoricheLMT13}.

% \newpage
\bibliographystyle{plain}
\bibliography{cwlower}

\newcommand{\comment}[1]{#1}
\comment{
\begin{appendix}
 \section{DNNF}\label{app:DNNF}
 \newcommand{\var}{\mathsf{var}}
\newcommand{\SB}{\{\,}
\newcommand{\SM}{\;{|}\;}
\newcommand{\SE}{\,\}}
 
 In this section of the appendix, we give a short primer on DNNFs. For more background, we recommend the very influential paper~\cite{DarwicheM02}.
 
 A \emph{(Boolean) circuit in negation normal form} (or \emph{NNF}) is a directed acyclic graph (DAG) with a single sink node (outdegree $0$) where each source node (indegree $0$) is labelled by a constant ($0$ or $1$) or by a literal, and each other node is labelled by $\land$ (AND) or $\lor$ (OR).
If $\varphi$ is an NNF and $v$ is a vertex of $\varphi$, the \emph{sub-NNF} of $\varphi$ rooted at $v$ is the NNF obtained from $\varphi$ by deleting every vertex from which $v$ cannot be reached along a directed path.
We write $\var(\varphi)$ for the set of variables occurring in an NNF $\varphi$.
Let $\varphi$ be an NNF and let $\tau$ be an assignment to $X \supseteq \var(\varphi)$.
Relative to $\tau$, we associate each vertex $v$ of $\varphi$ with a value $\mathit{val}_{\varphi}(v, \tau) \in \{0, 1\}$ as follows. 
If $v$ is labelled with a constant $c \in \{0, 1\}$ then $\mathit{val}_{\varphi}(v, \tau) = c$, and if $v$ is labelled with a literal $\ell$ then $\mathit{val}_{\varphi}(v, \tau) = \tau(\ell)$. 
If $v$ is an AND node then we let $\mathit{val}_{\varphi}(v, \tau) = \min \SB \mathit{val}_{\varphi}(w, \tau) \SM w$ is a child of~$v \SE$, and if $v$ is an OR node we define $\mathit{val}_{\varphi}(v, \tau) = \max \SB \mathit{val}_{\varphi}(w, \tau) \SM w$ is a child of $v \SE$. 
We say that $\tau$ \emph{satisfies} $\varphi$ if  $\mathit{val}_{\varphi}(s, \tau) = 1$, where~$s$ denotes the (unique) sink of $\varphi$. 
The function computed by $\varphi$ is defined in the obvious way.

An NNF $\varphi$ is \emph{decomposable} (in short, a \emph{DNNF}) if every AND node $v$ of $\varphi$ satisfies the following property: if $v$ has incoming edges from $v_1$ and $v_2$, and $\varphi_1$ and $\varphi_2$ denote the sub-NNFs of $\varphi$ rooted at $v_1$ and $v_2$, respectively, then $\var(\varphi_1)$ and $\var(\varphi_2)$ are disjoint. A DNNF $\varphi$ is \emph{deterministic} (a \emph{d-DNNF}) if, for every pair of distinct children $v_1$ and $v_2$ of an OR node, the sub-NNFs rooted at $v_1$ and $v_2$ do not have satisfying assignments in common.

DNNFs have first been defined by Darwiche~\cite{Darwiche01} and since then played a central role in knowledge compilation. There are several reasons for this success: It has been shown~\cite{DarwicheM02} that DNNFs can be seen as a generalization of many other successful representations that are used in knowledge compilation, in particular classical languages like OBDDs and FBDDs. Note however that DNNFs are in general far more succinct than these classical representation. Seeing known compilation languages as subclasses of DNNFs has several advantages. In particular it facilitates the structured and systematic comparison of different representations. Moreover, this systematic understanding allows to easily define new compilation languages as subclasses of DNNFs that have desirable properties for a task at hand. For example, the class of sentential decision diagrams~\cite{Darwiche11} is a subclass of DNNF defined in such a way as to have canonical representations, a property that often plays a central role in practice. Similarly, deterministic DNNFs are a restriction of DNNFs that is useful whenever efficient model counting is required. Arguably, seeing DNNFs as a unifying framework for the creation of compilation languages has been very influential in the field.

% Moreover, DNNFs allow several interesting queries efficiently, e.g.~consistency checks and clause entailment
 
 \section{Some communication complexity}\label{app:CC}

We introduce some very limited notions of communication complexity. The interested reader is referred to~\cite{KushilevitzN97} for general background and to~\cite{DurisHJSS04} for the rather specific notions and results on multi-partition communication complexity that appear in this paper.

Let $f$ be a boolean function defined on a set $X$ of $n$ boolean variables. Let $\Pi = (X_1, X_2)$ be a partition of $X$. We say that $\Pi$ is $\beta$-balanced for $\beta>0$ if $\min(|X_1|, |X_2) \ge \beta |X|$. The exact value of $\beta$ is unsubstantial for most of our considerations, and it is most of the time only necessary to assume that there is some constant $\beta$ for which all partitions that are considered are $\beta$-balanced. Consequently one often simply speaks of \emph{balanced partitions}.

A \emph{combinatorial rectangle} with respect to the partition $\Pi$ is a function $r: \{0,1\}^n\rightarrow \{0,1\}$ that can be written as $r^{(1)} \land r^{(2)}$, where the functions $r^{(1)}, r^{(2)}: \{0,1\}^n\rightarrow \{0,1\}$ depend only on the variables in $X_1$ and $X_2$, respectively. We define a \emph{(multi-partition) rectangle cover} of $f$ of size $t$ to be a set of rectangles $\{r_1, \ldots, r_t\}$ such that $f=r_1\lor r_2\lor \ldots \lor r_t$. Note that in a rectangle cover each rectangle may be with respect to its own balanced partition of the variables in~$X$.

Rectangle covers generalize similar notions of covers in in communication complexity by allowing different partitions for the different rectangles. They were originally defined to prove lower bounds for different types of branching programs in which the variables are not met in a fixed order. In~\cite{BovaCMS15} it was shown that lower bounds on the size of rectangle covers also give lower bounds for different versions of DNNFs, see e.g.~Theorem~\ref{thm:communicationDNNF}.

 \end{appendix}
}
\end{document}